%% file: main.tex
\documentclass[11pt]{article}
\input{macros}

\title{
Oracle-Efficient Online Classification \\ with Stochastic Inputs and Adversarial Outputs
}
\author{
\begin{tabular}{c@{\hspace{3cm}}c}
Gon Buzaglo$^{\dagger}$ & Elad Hazan$^{\dagger,\ddagger}$
\end{tabular}
\\[0.5em]
{\small $^{\dagger}$Department of Computer Science, Princeton University}\\
{\small $^{\ddagger}$Google DeepMind}
}

\date{October 5, 2026}

\begin{document}
\maketitle
\vspace{-1.4em}

\input{abstract}
\input{introduction}
\input{setting}
\input{algorithm}

\input{analysis}

\input{proofs}

\input{discussion}

\bibliographystyle{plainnat}
\bibliography{references}
\end{document}

%% file: macros.tex
\usepackage[margin=0.95in]{geometry}
\usepackage{amsmath,amssymb,amsthm,mathtools}
\usepackage{aliascnt}
\usepackage{microtype}
\usepackage{enumitem}
\usepackage{algorithm}
\usepackage{algorithmicx}
\usepackage{algpseudocode}
\usepackage{natbib}
\setcitestyle{round}

\usepackage[dvipsnames]{xcolor}
\definecolor{linkblue}{HTML}{3366CC}

\usepackage[
    colorlinks=true,
    linkcolor=linkblue,
    citecolor=linkblue,
    urlcolor=linkblue
]{hyperref}

\usepackage[nameinlink,noabbrev]{cleveref}

\setlist{nosep}

\newtheorem{theorem}{Theorem}
\newaliascnt{lemma}{theorem}
\newtheorem{lemma}[lemma]{Lemma}
\aliascntresetthe{lemma}
\newaliascnt{corollary}{theorem}

\aliascntresetthe{corollary}
\newaliascnt{proposition}{theorem}

\aliascntresetthe{proposition}
\theoremstyle{definition}
\newtheorem{definition}{Definition}
\newaliascnt{fact}{theorem}
\newtheorem{fact}[fact]{Fact}
\aliascntresetthe{fact}
\crefname{fact}{fact}{facts}
\Crefname{fact}{Fact}{Facts}

\newcommand{\E}{\mathbb E}
\newcommand{\Pp}{\mathbb P}
\newcommand{\R}{\mathbb R}
\newcommand{\1}{\mathbf 1}
\newcommand{\cH}{\mathcal H}
\newcommand{\cC}{\mathcal C}
\newcommand{\cD}{\mathcal D}
\newcommand{\DIS}{\operatorname{DIS}}
\newcommand{\VC}{\operatorname{VC}}
\DeclareMathOperator*{\argmin}{arg\,min}
\DeclareMathOperator*{\argmax}{arg\,max}

%% file: abstract.tex
\begin{abstract}
We consider binary prediction with i.i.d.\ contexts from an unknown
distribution and adaptively chosen losses. We show that a simple
Follow-the-Perturbed-Leader algorithm using a Gaussian perturbation for each
observed context achieves
$\widetilde O(\sqrt{T\log N})$ regret for a class of $N$ experts,
while requiring one optimization-oracle call per round and no explicit
enumeration of the class. For an infinite hypothesis class $\cH$, the same
algorithm achieves $\widetilde O(\sqrt{T\VC(\cH)})$ regret. This resolves an
open problem posed by \citet{lazaricmunos2012}, showing that hybrid
classification is computationally as easy as statistical learning.
As an application, we reduce the problem of contextual bandits with $K$ actions to
classification through uniform exploration, achieving $\widetilde O(K^{2/3}T^{2/3}(\log N)^{1/3})$ regret. This matches the best known dependence on the horizon while removing the
context-distribution access required by prior oracle-efficient methods.
\end{abstract}

%% file: introduction.tex
\section{Introduction}
\label{sec:intro}

The question we consider is whether stochastic contexts can make adversarial
online learning as computationally easy as statistical learning. We study this
in perhaps the most basic setting of contextual prediction with a
binary expert class $\cH\subseteq\{0,1\}^{\mathcal X}$, where each expert maps a
context to a binary prediction. On each round, a context is drawn i.i.d.\ from
an unknown distribution, while after observing the context and the past, an
adversary chooses losses for the two predictions. As usual in online learning,
the learner competes with the best expert in $\cH$ in hindsight. This
stochastic-input, adversarial-output model was introduced by
\citet{lazaricmunos2009,lazaricmunos2012}.

For a finite class $|\cH|=N$ that can be explicitly enumerated, classical
expert algorithms achieve $O(\!\sqrt{T\log N})$ regret
\citep{littlestonewarmuth1994,freundschapire1997}.
Their computational cost, however, scales with $N$, even though the regret
depends only logarithmically on $N$. In statistical learning, large hypothesis classes are instead accessed through
an empirical risk minimization (ERM) oracle
\citep{shalev2014understanding}. It is therefore natural to ask whether the same
$\sqrt{T\log N}$ regret can be achieved while accessing $\cH$ only through
such an oracle, without knowing or sampling from the context distribution.
More generally, for an infinite binary class, the corresponding
information-theoretic rate is
$\widetilde O(\sqrt{T\VC(\cH)})$
\citep{lazaricmunos2012,rakhlin2011stochastic,wu2023expected}.

This computational question was left open by
\citet{lazaricmunos2012}. \citet{wu2024} gave the first oracle-efficient
sublinear-regret algorithm for an unknown i.i.d.\ context distribution,
obtaining $\widetilde O(T^{3/4})$ regret for finite-VC classes. More recently,
\citet{okoroafor2026} obtained statistically optimal oracle-efficient rates
under the additional assumption that the adversary is restricted to a known
function class. Thus, for unrestricted adaptive losses, a gap remained between
the statistically optimal rate and what was known to be achievable using an
ERM oracle.

\paragraph{Our result.}
We close this gap. For a finite binary class $\cH$ of size $N$, a simple
Follow-the-Perturbed-Leader (FTPL) algorithm, \Cref{alg:ftpl}, satisfies
\[
\E R_T
=
O\!\left(\sqrt{T\log N}\log^2T\right)
\]
against arbitrary adaptive losses. The learner makes at most
one ERM call per round and requires neither knowledge of nor auxiliary samples
from the context distribution.

In \Cref{sec:vc-result} we show that for an infinite binary class with $1\le \VC(\cH)\le T$, \Cref{alg:ftpl} achieves
\[
\E R_T
=
O\!\left(
\sqrt{T\VC(\cH)\log(1+T/\VC(\cH))}\log^2T
\right)
=
\widetilde O(\sqrt{T\VC(\cH)})\,.
\]
Since the VC extension requires only a local modification of the analysis, we
first present the finite-class argument. As a further application, in
\Cref{alg:bandit} we reduce contextual bandits to classification using uniform exploration. For
a finite policy class of size $N$ over $K$ actions, this gives the new
oracle-efficient guarantee
\[
O\!\left(K^{2/3}T^{2/3}(\log N)^{1/3}\log^{4/3}T\right)\,,
\]
using $K$ value-oracle calls per round
and without requiring access to the context distribution, an assumption made by previous oracle-efficient algorithms with the same $T^{2/3}$ dependence, as we detail in \Cref{sec:bandit-result}.

\subsection{Related work}

\paragraph{Statistical and adversarial learning.}
Our setting lies between classical statistical learning and adversarial online
learning. Prediction with a finite class of experts admits
$O(\sqrt{T\log N})$ regret
\citep{freundschapire1997,cesabianchilugosi2006},
while statistical binary classification is characterized by the VC dimension
and admits efficient ERM-based procedures
\citep{shalev2014understanding,mohri2018foundations}.
In fully adversarial online classification, the corresponding complexity is
instead the Littlestone dimension and its sequential generalizations
\citep{littlestone1988,bendavidpalshalev2009,rakhlin2015sequential}.
Connections between stochastic and adversarial online learning also arise
through minimax characterizations of optimal regret
\citep{abernethy2009stochastic}.
Related interpolations between stochastic and adversarial learning include
models with constrained adversaries and hybrid stochastic--adversarial
processes \citep{rakhlin2011stochastic}. Other lines exploit predictable or
slowly varying sequences
\citep{hazankale2010,chiang2012gradual,rakhlin2013predictable}, or seek
``best-of-both-worlds'' guarantees that adapt between stochastic and
adversarial regimes
\citep{bubeckslivkins2012,seldinslivkins2014}.

\paragraph{Contextual and oracle-efficient learning.}
Contextual learning replaces individual experts by policies mapping contexts to
actions. See, e.g., \citet{hazan2022introduction} for a general treatment.
For an explicitly enumerable policy class, classical expert methods apply in
the full-information setting \citep{littlestonewarmuth1994,freundschapire1997}, while EXP4 gives the corresponding adversarial
bandit guarantee \citep{auer2002nonstochastic}. For large policy classes, contextual-bandit methods have been developed for
competing with rich classes of policies
\citep{beygelzimer2011contextual}. A substantial oracle-efficient literature
instead replaces explicit enumeration by supervised-learning or optimization
oracles
\citep{langfordzhang2007,dudik2011efficient,agarwal2014taming}. Related efficient approaches extend to adversarial contextual learning and
contextual online convex optimization
\citep{syrgkaniskrishnamurthyschapire2016,hazansingh2021boosting}.
For fully adversarial online learning, however, a standard offline optimization
oracle is insufficient in general \citep{hazankoren2016}, motivating stronger
oracle models and learning primitives
\citep{dudik2017oracle,hazanhu2018improper,agarwal2019learning}.

\paragraph{Additional structure on the context process.}
Several lines of work obtain oracle efficiency by giving the learner additional
information about the contexts or restricting how they are generated.
In the transductive setting, the future contexts, or their support, are known
in advance
\citep{kakadekalai2005,syrgkaniskrishnamurthyschapire2016}.
For i.i.d.\ contexts and adversarial losses, oracle-efficient contextual-bandit
algorithms are known when the context distribution is known or the learner can
draw auxiliary samples from it
\citep{rakhlin2016bistro,syrgkanis2016,banihashem2023}.
Another line considers smoothed contexts, where the adversary may choose the
context distribution adaptively subject to a density constraint
\citep{haghtalab2020smoothed,haghtalab2021adaptive,haghtalab2022oracle,
block2022,blockrakhlin2024,blanchard2025,blanchardshettyrakhlin2026}.
Finally, when both contexts and losses satisfy stochastic assumptions,
oracle-efficient contextual-bandit methods range from early reductions
\citep{langfordzhang2007} to statistically optimal algorithms
\citep{dudik2011efficient,agarwal2014taming}, with extensions to
nonstationary stochastic environments \citep{luo2018nonstationary,chen2019nonstationary}.
In contrast, our contexts are i.i.d.\ from an unknown distribution, no
additional contexts can be sampled, and the losses remain adaptively
adversarial.

\paragraph{Stochastic contexts and adversarial losses.}
The unknown-distribution hybrid model was introduced by
\citet{lazaricmunos2009,lazaricmunos2012}. Information-theoretically,
near-optimal rates are known for finite and finite-VC classes
\citep{lazaricmunos2012,rakhlin2011stochastic,wu2023expected}.
\citet{wu2024} gave the first oracle-efficient sublinear-regret guarantees,
obtaining $\widetilde O(T^{3/4})$ regret for finite-VC classes against adaptive
losses, while achieving the optimal $\widetilde O(\sqrt T)$ rate against
oblivious losses. In contrast with the noncontextual setting, where an
optimization oracle is insufficient in general even against an oblivious
adversary \citep{hazankoren2016}, obliviousness is a substantial assumption
here, as the loss functions are fixed before the sampled contexts are
observed.
While \citet{wu2024} also treat real-valued classes and shifting distributions, they identify adaptivity as
a genuine obstacle to their approach and explicitly pose the question of whether
an FTPL approach can be used in the unknown-distribution setting. More recently,
\citet{okoroafor2026} obtain oracle-efficient rates for adversaries restricted
to a known function class, with regret depending on the complexity of the
adversary class. In contrast, our FTPL approach with its nearly optimal guarantees allows unrestricted adaptive
losses.

\paragraph{Disagreement-based methods.}
Disagreement regions play a central role in active learning and selective
classification, where the mass of the disagreement region of the current
version space controls, respectively, label complexity and abstention
\citep{cohn1994improving,hanneke2007bound,dasguptahsumonteleoni2007,
hanneke2012,elyanivwiener2012,wiener2015compression}. Related disagreement-based complexity measures
characterize instance-dependent difficulty in contextual bandits
\citep{foster2021disagreement}. In our analysis, disagreement plays a different
role, as described immediately.

\paragraph{Techniques.}
The standard FTPL potential decomposition
\citep{kalai2005efficient} reduces regret to stability. In our setting,
perturbations are attached to observed contexts rather than directly to
hypotheses. Using the binary structure, we show that stability is controlled
by the gap between the best perturbed hypotheses predicting the two labels at
the current context. This reduces regret to the population mass of the
contexts on which nearly leading hypotheses disagree, which we formalize as disagreement in
\Cref{lem:regret-disagreement}.

While we can bound disagreement on the observed contexts using Gaussian symmetry and
Massart's lemma \citep{massart2000applications}, the main difficulty is to generalize
from this empirical bound to population disagreement, because the nearly
leading set depends on the observed contexts through the adaptive losses. Our
central generalization argument partitions the observed contexts according to
disagreement and perturbation magnitude and uses a Gaussian change of measure
to extract a valid estimate from a random subset of agreement coordinates.
This yields the population disagreement bound in \Cref{lem:population} and,
combined with the FTPL reduction, the stated regret guarantees. We give a more
detailed proof overview in \Cref{sec:analysis_overview}.

%% file: setting.tex
\section{Problem setting}
\label{sec:setting}

\paragraph{Online learning over hypotheses.}
Let $\cH\subseteq\{0,1\}^{\mathcal X}$ be a binary hypothesis class and
let $\cD$ be an unknown distribution over $\mathcal X$. On round $t$,
a fresh context $X_t\sim\cD$ is drawn and revealed. After
observing $X_t$ and the realized history, the adversary chooses a loss vector
$\ell_t\in[0,1]^2$. The learner then draws fresh private randomness, chooses
$a_t\in\{0,1\}$, incurs $\ell_t(a_t)$, and observes the full loss vector. The regret, defined below, compares the online predictions with the best
hypothesis in hindsight:
\begin{equation*}
R_T
=
\sum_{t=1}^T\ell_t(a_t)
-
\min_{h\in\cH}\sum_{t=1}^T\ell_t(h(X_t)).
\end{equation*}

\paragraph{ERM oracle.}
We use an ERM oracle over the absolute loss and real-valued labels, which receives $(X_1,y_1),\dots,(X_m,y_m)\in\mathcal{X}\times[0,1]$ and returns
\[\argmin_{h\in\cH}\sum_{i=1}^m |h(X_i)-y_i|.
\]
For $\mathcal{H}\subseteq\{0,1\}^\mathcal{X}$, it implements arbitrary weighted minimization: for arbitrary weights $w_1,\ldots,w_m\in\R$, let
$B=\max\{1,\max_i|w_i|\}$ and $y_i=(1-w_i/B)/2$. Since $h(X_i)\in\{0,1\}$, we have that
\[
\argmin_{h\in\mathcal H}
\sum_{i=1}^m |h(X_i)-y_i|
=
\argmin_{h\in\mathcal H}\{
\sum_{i=1}^m y_i
+
\frac{1}{B}\sum_{i=1}^m w_i h(X_i)\}
=
\argmin_{h\in\mathcal H}
\sum_{i=1}^m w_i h(X_i)\,.
\]
We fix a deterministic
tie-breaking rule and return an arbitrary fixed hypothesis on empty input.

\paragraph{Notation.}
We write $[m]=\{1,\ldots,m\}$, $|A|$ for the cardinality of a finite set
$A$, and $\1\{E\}$ for the indicator of an event $E$. All logarithms are
natural. We write $\mathcal{N}(\mu,\Sigma)$ for a Gaussian distribution and $I_k$ for the
$k\times k$ identity matrix. We write $\|\cdot\|$ for the Euclidean
norm of a vector. The notation
$O(\cdot)$ hides universal numerical constants, while $\widetilde O(\cdot)$
also suppresses logarithmic factors.

%% file: algorithm.tex
\section{Algorithms and results}
\label{sec:algorithm}
\Cref{alg:ftpl} resamples Gaussian perturbations on
observed contexts and follows the perturbed leader.

\begin{algorithm}[H]
\caption{Gaussian FTPL}
\label{alg:ftpl}
\begin{algorithmic}[1]
\State \textbf{Input:} horizon $T$, perturbation scale $\sigma>0$
\For{$t=1,\ldots,T$}
    \State Observe $X_t$
    \State Draw independent $G_{t,s}\sim \mathcal{N}(0,1)$ for $s<t$
    \State Choose $\displaystyle h_t\in\argmin_{h\in\cH}\left\{\sum_{s<t}\ell_s(h(X_s))-\sigma\sum_{s<t}G_{t,s}h(X_s)\right\}$
    \State Play $a_t=h_t(X_t)$ and observe $\ell_t$
\EndFor
\end{algorithmic}
\end{algorithm}

For $\cH\subseteq\{0,1\}^{\mathcal X}$, we have

\begin{align*}
\argmin_{h\in\cH}\{\sum_{s<t}\ell_s(h(X_s))-\sigma\sum_{s<t}G_{t,s}h(X_s)\}&=\argmin_{h\in\cH}\{
\sum_{s<t}\ell_s(0)
+
\sum_{s<t}\bigl(\ell_s(1)-\ell_s(0)-\sigma G_{t,s}\bigr)h(X_s)\}
\\
&=
\argmin_{h\in\cH}\{
\sum_{s<t}\bigl(\ell_s(1)-\ell_s(0)-\sigma G_{t,s}\bigr)h(X_s)\}\,,
\end{align*}
so each round requires one call to the absolute-loss ERM
oracle.

\begin{theorem}
\label{thm:main}
Let $T,N\ge2$ and let $\cH\subseteq\{0,1\}^\mathcal{X}$ be of size $N$. Algorithm~\ref{alg:ftpl}
with $\sigma=4\sqrt{\log T}$ satisfies
\begin{equation*}
\E R_T
=
O\!\left(\sqrt{T\log N}\,\log^2T\right).
\end{equation*}
\end{theorem}

\subsection{VC dimension result}
\label{sec:vc-result}

\begin{definition}[VC dimension]
A set of contexts $C=\{X_1,\ldots,X_m\}$ is \emph{shattered} by $\cH$ if
the restriction of $\cH$ to $C$ is the set of all functions from $C$ to
$\{0,1\}$. The VC dimension $\VC(\cH)$ is the largest $m$ for which such
a set exists.
\end{definition}

VC dimension is known to measure the statistical learnability of an infinite hypothesis class (cf. \citet{shalev2014understanding}). Here we show that it also characterizes the efficient learnability of an infinite class in the hybrid setting.

\begin{theorem}
\label{thm:vc-extension}
For $T\ge2$ and $\cH\subseteq\{0,1\}^{\mathcal X}$ with $\VC(\cH)\in[T]$, 
\Cref{alg:ftpl} with
$\sigma=4\sqrt{\log T}$ satisfies
\begin{equation*}
\E R_T
=
O\!\left(
\sqrt{T\VC(\cH)\log(1+T/\VC(\cH))}\,\log^2T
\right).
\end{equation*}
\end{theorem}

\input{contextual_bandit_algorithm}

%% file: contextual_bandit_algorithm.tex
\subsection{Contextual bandits}
\label{sec:bandit-result}

We next apply our binary result to contextual bandits with a finite policy
class over $K$ actions by associating each policy $h$ with the binary hypothesis
$f_h:\mathcal X\times[K]\to\{0,1\}$ defined by
$f_h(X,A)=\1\{h(X)=A\}$. This reduces the contextual-bandit problem to binary
classification, at the price of $T^{2/3}$ regret rather than $\sqrt T$.

\paragraph{Protocol and regret.}
We consider sequential decision making over a set of $K$ actions using a finite policy class $\mathcal{H}\subseteq [K]^{\mathcal{X}}$ of size $N$. On round
$t\in[T]$, a fresh context $X_t\sim\cD$ is observed. The adversary then chooses
$c_t\in[0,1]^K$ after observing $X_t$, the realized past, and the
conditional distribution of the decision maker over her current action, but
before she draws any fresh randomness. The decision maker then draws her
private randomness, plays $A_t\in[K]$, and observes only $c_t(A_t)$. As in
\citet{banihashem2023}, this adversary is adaptive but nonanticipating. The regret is
\[
R_T^{\mathrm{bandit}}
=
\sum_{t=1}^T c_t(A_t)
-
\min_{h\in\cH}\sum_{t=1}^T c_t(h(X_t)).
\]

\paragraph{ERM-value oracle.}
Following prior work \citep{syrgkanis2016,banihashem2023,wu2024}, we use a
cost-sensitive value oracle. Given contexts $x_i$ and cost vectors $v_i\in\R^K$, rather than returning a minimizing policy
the oracle returns only the value
\[
\min_{h\in\cH}\sum_i v_i(h(x_i))\,.
\]

\begin{algorithm}[H]
\caption{Gaussian FTPL with uniform exploration}
\label{alg:bandit}
\begin{algorithmic}[1]
\State \textbf{Input:} horizon $T$, exploration probability $\gamma\in(0,1]$, perturbation scale $\sigma>0$
\State Initialize $\mathcal E_0\gets\emptyset$
\For{$t=1,\ldots,T$}
    \State Observe $X_t$
    \State Draw $B_t\sim\operatorname{Bernoulli}(\gamma)$
    \If{$B_t=1$}
        \State Draw $A_t\sim\operatorname{Unif}([K])$
        \State $\mathcal E_t\gets\mathcal E_{t-1}\cup\{t\}$
    \Else
        \State Draw independent $G_{t,s}\sim \mathcal{N}(0,1)$ for every $s\in\mathcal E_{t-1}$
        \For{$a\in[K]$}
            \State $\displaystyle
            V_t(a)\gets\min_{h\in\cH}\left\{
            \1\{h(X_t)\ne a\}+\sum_{s\in\mathcal E_{t-1}}
            \bigl(c_s(A_s)-\sigma G_{t,s}\bigr)\1\{h(X_s)=A_s\}
            \right\}$
        \EndFor
        \State Set $\displaystyle A_t\gets\min\argmin_{a\in[K]}V_t(a)$
        \State $\mathcal E_t\gets\mathcal E_{t-1}$
    \EndIf
    \State Observe $c_t(A_t)$
\EndFor
\end{algorithmic}
\end{algorithm}

\begin{theorem}
\label{thm:bandit}
Let $K,T,N\ge2$ and $\cH\subseteq[K]^{\mathcal X}$ be a policy class of size
$N$. Set $\sigma=4\sqrt{\log T}$ and
\[
\gamma
=
\min\left\{
1,
\left(\frac{K^2\log N\,\log^4T}{T}\right)^{1/3}
\right\}.
\]
Then \Cref{alg:bandit} makes at most $K$ calls to the cost-sensitive value oracle per
round and satisfies
\[
\E R_T^{\mathrm{bandit}}
=
O\!\left(K^{2/3}T^{2/3}(\log N)^{1/3}\log^{4/3}T\right).
\]
\end{theorem}

The significance of \Cref{thm:bandit} is that our classification method
can be used to remove distribution access from oracle-efficient adversarial
contextual bandits. Previous approaches attaining $T^{2/3}$ regret for i.i.d.\
contexts rely on knowledge of the context distribution or auxiliary samples
from it \citep{rakhlin2016bistro,syrgkanis2016,banihashem2023}. The best such
guarantee, due to \citet{banihashem2023}, is
$O(T^{2/3}(K\log N)^{1/3})$. Our reduction requires neither resource: uniform
exploration makes the augmented contexts $(X_t,A_t)$ on exploration rounds
i.i.d.\ from $\cD\times\operatorname{Unif}([K])$, so our binary FTPL guarantee
uses only contexts encountered during the interaction.

For an unknown context distribution, the previous oracle-efficient guarantee,
due to \citet{wu2024}, is
$\widetilde O((K^{2/3}(\log N)^{1/3}+K)T^{4/5})$ regret. \Cref{tab:bandit-comparison}
compares the guarantees and assumptions.

\begin{table}[H]
\centering
\footnotesize
\setlength{\tabcolsep}{3.5pt}
\setlength{\abovecaptionskip}{2pt}
\setlength{\belowcaptionskip}{2pt}
\begin{tabular}{@{}lccc@{}}
\hline
& Distribution access & Adaptivity & Regret \\
\hline
\rule{0pt}{2.5ex}\citet{banihashem2023}
& Yes
& nonanticipating
& $\widetilde O(K^{1/3}T^{2/3})$ \\
\citet{wu2024}
& No
& contexts
& $\widetilde O(KT^{4/5})$ \\
\Cref{thm:bandit}
& No
& nonanticipating
& $\widetilde O(K^{2/3}T^{2/3})$ \\
\hline
\end{tabular}
\caption{Comparison with the closest prior oracle-efficient contextual-bandit
  guarantees, suppressing logarithmic factors in $N$ and $T$.
  The adversary in \citet{wu2024} is adaptive to the sampled contexts but
  oblivious to the realized past. Thus, without distribution access,
  \Cref{thm:bandit} improves the regret against the stronger adversary.}
\label{tab:bandit-comparison}
\end{table}

%% file: analysis.tex
\section{Analysis}
\label{sec:analysis}
We first explain the proof ideas behind
\Cref{thm:main,thm:vc-extension,thm:bandit} and then state the main technical results used in
their analysis. The proofs are deferred to \Cref{sec:proofs}.

\subsection{Proof ideas}
\label{sec:analysis_overview}

The standard FTPL analysis is organized around \emph{stability}: how much
incorporating the current loss changes the perturbed leader. By defining the potential
\[
\Phi_t(L)
=
\E_{G_1,\ldots,G_t}
\max_{h\in\cH}
\left\{
-L(h)+\sigma\sum_{s\le t}G_sh(X_s)
\right\}\,,
\]
the usual telescoping argument gives
\begin{align*}
\E R_T
&\leq
\E\sum_{t=1}^T
\left[
\ell_t(a_t)+\Phi_t(L_t)-\Phi_{t-1}(L_{t-1})
\right]
\\
&=
\E\sum_{t=1}^T
\left[
\underbrace{
\ell_t(a_t)
+\Phi_{t-1}(L_t)-\Phi_{t-1}(L_{t-1})
}_{\text{stability}}
+
\underbrace{
\Phi_t(L_t)-\Phi_{t-1}(L_t)
}_{\text{perturbation term}}
\right].
\end{align*}

Unlike the simpler noncontextual experts problem, where perturbations are attached
directly to the experts \citep{kalai2005efficient}, here they are highly correlated across different experts.
We therefore do not bound the two
terms above separately, but rather control their sum directly
by the volume of contexts on which nearly leading hypotheses disagree. \Cref{lem:regret-disagreement} reduces the regret to
\[
\sum_{t=2}^T
\int_0^\infty
\Pp(G\ge s/\sigma)\,
\E\!\left[\cD(\DIS(A_{t-1,s+1}))\right]\,ds.
\]

 We first bound the number of observed contexts on
which they disagree using Gaussian symmetry and Massart's lemma, obtaining
\[
\E_{G_1,\ldots,G_n}\sum_{i=1}^n\1\{X_i\in\DIS(A_{n,r})\}
\le
2\sqrt{ne^{r^2/\sigma^2}\log N}\,.
\]
The main
difficulty is then to generalize this empirical disagreement bound to the
population, since the nearly leading set depends on the observed contexts
through the adaptive losses. Our change-of-measure argument identifies a
subset of observed contexts on which nearly leading hypotheses agree and receive large perturbations as a valid sample to form an estimate of the disagreement volume,
yielding \Cref{lem:transfer}. Combining the empirical and population bounds
gives \begin{equation*}
\E\bigl[\cD(\DIS(A_{n,r}))\bigr]
\le
C(1+\log n)
\sqrt{\frac{e^{r^2/\sigma^2}\log N}{n}}
\end{equation*} in \Cref{lem:population}. For infinite classes, the same argument is
localized to the labelings realized on the observed contexts and the finite
cardinality $N$ is replaced by the growth function $\Pi_{\cH}$. The contextual bandits application is obtained by reducing exploitation to FTPL over the exploration rounds.

\subsection{Defining disagreement on nearly leading hypotheses}
\label{sec:finite-analysis}

Given $n\ge1$, contexts $X_1,\ldots,X_n$, loss vectors
$\ell_1,\ldots,\ell_n$, and independent standard Gaussians
$G_1,\ldots,G_n$, define, for $r\ge0$,
\begin{equation*}
A_{n,r}
=
\left\{
 h:
 \sum_{i=1}^n\bigl[\ell_i(h(X_i))-\sigma G_i h(X_i)\bigr]
 \le
 \min_{g\in\mathcal H}\sum_{i=1}^n\bigl[\ell_i(g(X_i))-\sigma G_i g(X_i)\bigr]+r
\right\}.
\end{equation*}
For a set of hypotheses $H\subseteq\mathcal{H}$, let
$\DIS(H)\subseteq\mathcal{X}$ be its disagreement region,
$$\DIS(H)=\{X\in \mathcal{X}:\exists h,g\in H ~~~\text{s.t}~~~\ h(X)\ne g(X)\}\,.$$

\subsection{Reducing stability to disagreement}

The first step of the analysis reduces regret to the population mass of the
disagreement region.

\begin{lemma}
\label{lem:regret-disagreement}
For Algorithm~\ref{alg:ftpl} with $\sigma\ge\sqrt{2\pi}$,
\begin{equation*}
\E R_T
\le
2\sigma
+
2\sum_{t=2}^T\int_0^\infty
\Pp_{G\sim \mathcal{N}(0,1)}(G\ge s/\sigma)\,
\E\bigl[\cD(\DIS(A_{t-1,s+1}))\bigr]\,ds.
\end{equation*}
\end{lemma}

The proof combines the standard FTPL potential decomposition with a binary
stability argument. At a fixed context, the gap between the best perturbed
loss values for the two predictions controls both the probability that a fresh
Gaussian perturbation changes the leader and whether nearly leading hypotheses
disagree. Integrating this relation yields the bound above, as worked out in \Cref{lem:binary-stability-disagreement}.

\subsection{Disagreement analysis}

To use \Cref{lem:regret-disagreement},
we prove the following.

\begin{lemma}
\label{lem:population}
There is a universal constant $C\ge1$ such that, for every $\sigma\ge2$ and
$r\ge0$,
\begin{equation*}
\E\bigl[\cD(\DIS(A_{n,r}))\bigr]
\le
C(1+\log n)
\sqrt{\frac{e^{r^2/\sigma^2}\log N}{n}}
.
\end{equation*}
\end{lemma}

To prove \Cref{lem:population}, we first control empirical disagreement.

\begin{lemma}
\label{lem:empirical}
For every $r\ge0$,
\[
\E_{G_1,\ldots,G_n}\sum_{i=1}^n\1\{X_i\in\DIS(A_{n,r})\}
\le
2\sqrt{ne^{r^2/\sigma^2}\log N}\,.
\]
\end{lemma}

The proof turns the event $X_i\in\DIS(A_{n,r})$ into a Rademacher process
and applies Massart's lemma (cf. \citet[Theorem~3.7]{mohri2018foundations}), restated here.

\begin{lemma}[Massart]
\label{lem:massart}
Let $\mathcal A\subset\R^n$ be finite and let $\varepsilon_1,\ldots,\varepsilon_n$ be independent random variables uniformly distributed on $\{-1,+1\}$. Then
\[
\E_{\varepsilon_1,\ldots,\varepsilon_n}\max_{v\in\mathcal A}\sum_{i=1}^n\varepsilon_i v_i
\le \sqrt{2\log|\mathcal A|}\max_{v\in\mathcal A}\|v\|.
\]
\end{lemma}

The second ingredient generalizes this
empirical control to population disagreement.

\begin{lemma}
\label{lem:transfer}
Let $p=\Pp(G>1/\sigma)$ for $G\sim \mathcal{N}(0,1)$. For every
$k\in\{0,\ldots,n\}$ and $\varepsilon\in(0,1)$,
\begin{equation*}
\Pp\!\left(
\sum_{i=1}^n\1\{X_i\in\DIS(A_{n,r})\}=k,
\ \cD(\DIS(A_{n,r}))>\varepsilon
\right)
\le
\binom nk N e^{k/(2\sigma^2)}
\left(1-p+p\sqrt{1-\varepsilon}\right)^{n-k}.
\end{equation*}
\end{lemma}

\Cref{lem:transfer} is the most delicate part of the analysis. We union bound
over leading hypotheses and empirical disagreement sets, and use a
change-of-measure argument to relate the disagreement set to a fictitious
surrogate set determined only by contexts on which nearly leading hypotheses
agree and whose perturbations are large. The remaining observed contexts are
then independent of the surrogate set and can be used to estimate its
population mass.

\subsection{VC dimension analysis}
\label{sec:vc-extension}

The growth function, defined below, measures the ``effective'' size of a hypothesis class.

\begin{definition}[Growth function]
\label{def:growth}
For a binary hypothesis class $\cH$ and $n\in\mathbb N$, define
\[
\Pi_{\cH}(n)
=
\sup_{X_1,\ldots,X_n\in\mathcal{X}}
\left|
\left\{
(h(X_1),\ldots,h(X_n)):h\in\cH
\right\}
\right|\,.
\]
\end{definition}
Thus $\Pi_{\cH}(n)$ is the maximum number of distinct labelings induced by
$\cH$ on $n$ contexts. It is known to be controlled by $\VC$ dimension via the
Sauer--Shelah--Perles lemma \citep{sauer1972density,shelah1972combinatorial}.

\begin{lemma}[Sauer--Shelah--Perles]
\label{lem:sauer}
If $1\le\VC(\mathcal H)\le m$, then
\[
\log\Pi_{\mathcal H}(m)
\le
2\VC(\mathcal H)\log\left(1+\frac{m}{\VC(\mathcal H)}\right).
\]
\end{lemma}

The VC result follows from observing that we analyze labeling of observed contexts rather than arbitrary properties of the class. However, to prove an analogue of \Cref{lem:transfer}, the intuition
``where perturbations are large'' needs to be treated carefully, as we first
need to choose a subset of the observed contexts whose labeling we will count. This is formalized in \Cref{sec:vc-proofs}, yielding the following result.

\begin{lemma}
\label{lem:vc-population}
There is a universal constant $C\ge1$ such that, for every $\sigma\ge2$ and
$r\ge0$,
\begin{equation*}
\E\bigl[\cD(\DIS(A_{n,r}))\bigr]
\le
C(1+\log n)
\sqrt{\frac{\log\Pi_{\cH}(n)}{n}}
\exp\left(\frac{r^2}{2\sigma^2}\right).
\end{equation*}
\end{lemma}

Together with \Cref{lem:regret-disagreement,lem:sauer},
\Cref{lem:vc-population} yields \Cref{thm:vc-extension}, as we detail in
\Cref{sec:vc-proofs}.

\input{contextual_bandit_analysis}

%% file: contextual_bandit_analysis.tex
\subsection{Contextual bandits reduction}
\label{sec:bandit-analysis}

The proof has two parts. On exploration rounds, the augmented contexts
$(X_t,A_t)$ are i.i.d.\ from $\cD\times\operatorname{Unif}([K])$, and the
update rule on exploitation rounds becomes \Cref{alg:ftpl} on the binary class
\[
\mathcal F
=
\left\{
f_h:\mathcal X\times[K]\to\{0,1\}
:
f_h(X,A)=\1\{h(X)=A\},
\ h\in\cH
\right\}.
\]
The following lemma then converts the FTPL regret on the importance-weighted
exploration losses $\widehat\ell_t$ to regret on the full bandit sequence.

\begin{lemma}
\label{lem:bandit-reduction}
For every $t\in[T]$ and $s\in\mathcal E_{t-1}$, draw
$G_{t,s}\sim \mathcal{N}(0,1)$ and define
\[
h_t
\in
\argmin_{h\in\cH}
\sum_{s\in\mathcal E_{t-1}}
\bigl(c_s(A_s)-\sigma G_{t,s}\bigr)\1\{h(X_s)=A_s\},
\]
with ties broken so that $h_t(X_t)$ is minimal. Define
\[
\widehat\ell_t(h)
=
\frac K\gamma B_t c_t(A_t)\1\{h(X_t)=A_t\}.
\]
Then, for a universal constant $C\ge1$,
\begin{align*}
\E R_T^{\mathrm{bandit}}
\le
\gamma T
+
\E\left[
\sum_{t=1}^T\widehat\ell_t(h_t)
-
\min_{h\in\cH}\sum_{t=1}^T\widehat\ell_t(h)
\right]
+
C\left(
\sqrt{\frac{KT\log N}{\gamma}}
+
\frac{K\log N}{\gamma}
\right).
\end{align*}
\end{lemma}

%% file: proofs.tex
\section{Proofs}
\label{sec:proofs}
\begin{proof}[Proof of \Cref{thm:main}]
By \Cref{lem:regret-disagreement}, we have that
\begin{align*}
\E R_T
&\le
2\sigma
+
2\sum_{t=2}^T
\int_0^\infty
\Pp(G\ge s/\sigma)\,
\E\bigl[\cD(\DIS(A_{t-1,s+1}))\bigr]\,ds\,.
\end{align*}
Then, by \Cref{lem:population} and $\cD(\DIS(A_{t-1,s+1}))\leq1$ we can write
\begin{align*}
\E R_T
&\le
2\sigma
+
2\sum_{n=1}^{T-1}
\int_0^\infty
\Pp(G\ge s/\sigma)
\min\left\{
1,
C(1+\log n)
\sqrt{\frac{\log N}{n}}
\exp\left(\frac{(s+1)^2}{2\sigma^2}\right)
\right\}\,ds
\\
&\le
2\sigma
+
\sum_{n=1}^{T-1}
\int_0^\infty
e^{-s^2/(2\sigma^2)}
\min\left\{
1,
C(1+\log n)
\sqrt{\frac{\log N}{n}}
\exp\left(\frac{(s+1)^2}{2\sigma^2}\right)
\right\}\,ds
\\
&\le
2\sigma
+
\sum_{n=1}^{T-1}
\left[
C(1+\log n)\sqrt{\frac{\log N}{n}}
\int_0^{\sigma^2/2}2ds
+
\int_{\sigma^2/2}^\infty
e^{-s^2/(2\sigma^2)}\,ds
\right]
\\
&\le
2\sigma
+
\sum_{n=1}^{T-1}
\left[
C(1+\log n)\sqrt{\frac{\log N}{n}}\,\sigma^2
+
2e^{-\sigma^2/8}
\right]
\\
&=
2\sigma
+
16C\log T\sqrt{\log N}
\sum_{n=1}^{T-1}\frac{1+\log n}{\sqrt n}
+
\frac{2(T-1)}{T^2}
\\
&=
O\!\left(\sqrt{T\log N}\,\log^2 T\right)\,.
\end{align*}
\end{proof}

\input{regret-disagreement-lemma-proof}

\subsection{Disagreement analysis}
\begin{proof}[Proof of \Cref{lem:population}]
Write
\(
K_{n,r}=\sum_{i=1}^n\1\{X_i\in\DIS(A_{n,r})\}\) and 
\(
V_{n,r}=\cD(\DIS(A_{n,r}))\,,
\)
and also let $p=\Pp_{G\sim \mathcal{N}(0,1)}(G>1/\sigma)$,
$k\le n/2$, and 
\[
\varepsilon_k
=
\frac{4}{pn}
\left(
\log N+\log\binom nk+\frac{k}{2\sigma^2}
+2\log n
\right)\,.
\]
If $\varepsilon_k\ge1$, then $\Pp(K_{n,r}=k,V_{n,r}>\varepsilon_k)=0$.
Otherwise, \Cref{lem:transfer} gives
\begin{align*}
\Pp(K_{n,r}=k,V_{n,r}>\varepsilon_k)
&\le
\binom nk N e^{k/(2\sigma^2)}
\left(1-p(1-\sqrt{1-\varepsilon_k})\right)^{n-k}
\\
&\le
\exp\left(
\log N+\log\binom nk+\frac{k}{2\sigma^2}
-p(n-k)(1-\sqrt{1-\varepsilon_k})
\right)
\\
&\le
\exp\left(
\log N+\log\binom nk+\frac{k}{2\sigma^2}
-\frac{pn\varepsilon_k}{4}
\right)
\\
&
\leq
\frac{1}{n^2}\,,
\end{align*}
where we used $\log(1-x)\le-x$,
$1-\sqrt{1-\varepsilon}\ge\varepsilon/2$ and $n-k\ge n/2$. Thus we have
\begin{align*}
\E V_{n,r}
&=
\sum_{k\le n/2}
\E\!\left[
V_{n,r}\1\{K_{n,r}=k,V_{n,r}\le\varepsilon_k\}
\right]
\\
&\quad+
\sum_{k\le n/2}
\E\!\left[
V_{n,r}\1\{K_{n,r}=k,V_{n,r}>\varepsilon_k\}
\right]
+
\E\!\left[
V_{n,r}\1\{K_{n,r}>n/2\}
\right]
\\
&\le
\sum_{k\le n/2}
\varepsilon_k\Pp(K_{n,r}=k)
+
\sum_{k\le n/2}
\Pp(K_{n,r}=k,V_{n,r}>\varepsilon_k)
+
\Pp(K_{n,r}>n/2)
\\
\text{(Markov)}
\qquad
&\le
\sum_{k\le n/2}
\frac{4}{pn}
\left(
\log N+\log\binom nk+\frac{k}{2\sigma^2}
+2\log n
\right)\Pp(K_{n,r}=k)
+
\frac1n
+
\frac2n\E K_{n,r}
\\
&\le
C\left(
\frac{\log N+\log n}{n}
+
\frac{1+\log n}{n}\E K_{n,r}
\right),
\end{align*}
where $C\ge1$ is a universal constant and we used
$\log\binom nk\le k\log n$ and $\sigma\ge2$. By \Cref{lem:empirical},
\[
\E V_{n,r}
\le
C\left(
\frac{\log N+\log n}{n}
+
(1+\log n)\sqrt{\frac{\log N}{n}}
\exp\left(\frac{r^2}{2\sigma^2}\right)
\right)\,.
\]
Since $V_{n,r}\le1$, the first term is absorbed into the second after
increasing $C$.
\end{proof}

\input{empirical_lemma_proof}

\input{transfer_lemma_proof}

\subsection{VC dimension proofs}
\label{sec:vc-proofs}

\begin{lemma}
\label{lem:vc-empirical}
For every $r\ge0$,
\begin{equation*}
\E_G\sum_{i=1}^n\1\{X_i\in\DIS(A_{n,r})\}
\le
2\sqrt{n\log\Pi_{\cH}(n)}
\exp\left(\frac{r^2}{2\sigma^2}\right)\,.
\end{equation*}
\end{lemma}

\begin{proof}
Let
\(
I=\{i\in[n]:\{h(X_i):h\in\cH\}=\{0,1\}\}.
\)
As in the proof of \Cref{lem:empirical},
\begin{align*}
\E_G\sum_{i=1}^n\1\{X_i\in\DIS(A_{n,r})\}
&\le
2e^{-r^2/(2\sigma^2)}
\E_G
\E_{\xi_1,\ldots,\xi_n\stackrel{\mathrm{iid}}{\sim}\operatorname{Unif}\{-1,+1\}}
\max_{h\in\cH}
\sum_{i\in I}
\xi_i
\sinh\left(\frac r\sigma |G_i|\right)
h(X_i)
\\
\text{(\Cref{lem:massart})}\qquad
&\le
2e^{-r^2/(2\sigma^2)}
\E_G
\sqrt{
2\log\Pi_{\cH}(n)
\sum_{i\in I}
\sinh^2\left(\frac r\sigma |G_i|\right)
}
\\
\text{(Jensen)}\qquad
&\le
2e^{-r^2/(2\sigma^2)}
\sqrt{
2n\log\Pi_{\cH}(n)\,
\E_G\sinh^2\left(\frac r\sigma G\right)
}
\\
\text{(Gaussian MGF)}\qquad
&=
2e^{-r^2/(2\sigma^2)}
\sqrt{
n\log\Pi_{\cH}(n)
\left(e^{2r^2/\sigma^2}-1\right)
}
\\
&\le
2\sqrt{n\log\Pi_{\cH}(n)}
\exp\left(\frac{r^2}{2\sigma^2}\right)\,.
\end{align*}
\end{proof}

\begin{lemma}
\label{lem:vc-transfer}
Let $p=\Pp(G>1/\sigma)$ for $G\sim \mathcal{N}(0,1)$. For every
$k\in\{0,\ldots,n\}$ and $\varepsilon\in(0,1)$,
\begin{equation}
\Pp\!\left(
\sum_{i=1}^n\1\{X_i\in\DIS(A_{n,r})\}=k,
\ \cD(\DIS(A_{n,r}))>\varepsilon
\right)
\le
\binom nk\Pi_{\cH}(n)e^{k/(2\sigma^2)}
\left(1-p+p\sqrt{1-\varepsilon}\right)^{n-k}\,.
\label{eq:vc-transfer-tail}
\end{equation}
\end{lemma}

To prove \Cref{lem:vc-transfer}, we can no longer apply a union bound over a finite hypothesis class, and thus should not study a fixed hypothesis $h$ but rather a fixed labeling of the contexts, which will give the relation to the growth function. However, the held out contexts were previously chosen by the fixed hypothesis, so to avoid a circular argument we observe that agreeing with the leading hypothesis where its perturbations are large implies agreeing everywhere:

\begin{lemma}
\label{lem:vc-representative}
Fix $B\subseteq[n]$ and $\widehat h\in A_{n,0}$, and let
\(
P
=
\left\{
i\notin B:
(1-2\widehat h(X_i))G_i>\frac1\sigma
\right\}.
\)
If $h\in\cH$ satisfies
\(
h(X_i)=\widehat h(X_i)
\)
for every $i\notin P$, then
\(
h(X_i)=\widehat h(X_i)
\) for every $i\in[n]$.
\end{lemma}

\begin{proof}
Since $h(X_i)=\widehat h(X_i)$ for every $i\notin P$,
\begin{align*}
0
&\ge
\sum_{i=1}^n
\left[
\ell_i(\widehat h(X_i))-\sigma G_i\widehat h(X_i)
-\ell_i(h(X_i))+\sigma G_i h(X_i)
\right]
\\
&=
\sum_{i\in P}
\bigl(\sigma G_i-\ell_i(1)+\ell_i(0)\bigr)
(1-2\widehat h(X_i))
\1\{h(X_i)\ne\widehat h(X_i)\}
\\
&\ge0\,.
\end{align*}
Each nonzero summand is strictly positive by the definition of $P$, so
$h(X_i)=\widehat h(X_i)$ for every $i\in P$.
\end{proof}

\begin{proof}[Proof of \Cref{lem:vc-transfer}]
For every $B\subseteq[n]$, $|B|=k$, and
$P\subseteq[n]\setminus B$, let
\(
\mathcal U_P
=
\left\{
\bigl(h(X_i)\bigr)_{i\notin P}:h\in\cH
\right\},
\)
so $|\mathcal U_P|\le\Pi_{\cH}(n)$, and, for every
$u\in\mathcal U_P$, choose some $h_u\in\cH$ with
\(
(h_u(X_i))_{i\notin P}=u
\).
For $u\in\mathcal U_P$, define the event
\[
\mathcal F_{B,P,u}
=
\left\{
P=
\left\{
i\notin B:
(1-2h_u(X_i))G_i>\frac1\sigma
\right\}
\right\}.
\]
On the event
\(
\sum_{i=1}^n\1\{X_i\in\DIS(A_{n,r})\}=k,
\)
write
\(
B^\star=\{i:X_i\in\DIS(A_{n,r})\},
\)
choose $\widehat h\in A_{n,0}$, and write
\(
P^\star
=
\left\{
i\notin B^\star:
(1-2\widehat h(X_i))G_i>\frac1\sigma
\right\}.
\)
For
\(
u=(\widehat h(X_i))_{i\notin P^\star}\in\mathcal U_{P^\star}
\),
\Cref{lem:vc-representative} gives
\(
h_u(X_i)=\widehat h(X_i)
\)
for every $i\in[n]$. Hence $h_u\in A_{n,0}$ and
$\mathcal F_{B^\star,P^\star,u}$ holds. Therefore,
\begin{align*}
&\Pp\left(
\sum_{i=1}^n\1\{X_i\in\DIS(A_{n,r})\}=k,
\ \cD(\DIS(A_{n,r}))>\varepsilon
\right)
\\
\text{(union bound)}\qquad
&\le
\sum_{\substack{B\subseteq[n]\\|B|=k}}
\sum_{P\subseteq[n]\setminus B}
\E
\sum_{u\in\mathcal U_P}
\Pp\left(
\begin{gathered}
h_u\in A_{n,0},\
B=\{i:X_i\in\DIS(A_{n,r})\},\\
\mathcal F_{B,P,u},\
\cD(\DIS(A_{n,r}))>\varepsilon
\end{gathered}
\,\middle|\,
(X_i)_{i\notin P}
\right)
\\
\text{(\Cref{lem:suboptimal-set-characterization})}\qquad
&\le
\sum_{\substack{B\subseteq[n]\\|B|=k}}
\sum_{P\subseteq[n]\setminus B}
\E
\sum_{u\in\mathcal U_P}
\Pp\left(
\begin{gathered}
\mathcal F_{B,P,u},\
\cD(\DIS(\cC_W))>\varepsilon,\\
X_i\notin\DIS(\cC_W)\quad\forall i\in P
\end{gathered}
\,\middle|\,
(X_i)_{i\notin P}
\right),
\end{align*}

For fixed $B,P,u$, by symmetry and independence of the Gaussian
perturbations, we have that
\(
\Pp\left(\mathcal F_{B,P,u}\mid (X_i)_{i\notin P}\right)
=
p^{|P|}(1-p)^{n-k-|P|}.
\)
Hence, applying
\Cref{lem:gaussian-holdout} with $h=h_u$ yields
\begin{align*}
&\Pp\left(
\begin{gathered}
\mathcal F_{B,P,u},\
\cD(\DIS(\cC_W))>\varepsilon,\\
X_i\notin\DIS(\cC_W)\quad\forall i\in P
\end{gathered}
\,\middle|\,
(X_i)_{i\notin P}
\right)
\\
&\qquad=
p^{|P|}(1-p)^{n-k-|P|}
\Pp\left(
\begin{gathered}
\cD(\DIS(\cC_W))>\varepsilon,\\
X_i\notin\DIS(\cC_W)\quad\forall i\in P
\end{gathered}
\,\middle|\,
\mathcal F_{B,P,u},(X_i)_{i\notin P}
\right)
\\
&\qquad\le
p^{|P|}(1-p)^{n-k-|P|}
e^{k/(2\sigma^2)}(1-\varepsilon)^{|P|/2}\,.
\end{align*}

Consequently,
\begin{align*}
&\Pp\left(
\sum_{i=1}^n\1\{X_i\in\DIS(A_{n,r})\}=k,
\ \cD(\DIS(A_{n,r}))>\varepsilon
\right)
\\
&\qquad\le
\sum_{\substack{B\subseteq[n]\\|B|=k}}
\sum_{P\subseteq[n]\setminus B}
\E\sum_{u\in\mathcal U_P}
p^{|P|}(1-p)^{n-k-|P|}
e^{k/(2\sigma^2)}(1-\varepsilon)^{|P|/2}
\\
&\qquad\le
\sum_{\substack{B\subseteq[n]\\|B|=k}}
\sum_{P\subseteq[n]\setminus B}
\Pi_{\cH}(n)
p^{|P|}(1-p)^{n-k-|P|}
e^{k/(2\sigma^2)}(1-\varepsilon)^{|P|/2}
\\
&\qquad=
\binom nk
\Pi_{\cH}(n)e^{k/(2\sigma^2)}
\sum_{j=0}^{n-k}
\binom{n-k}{j}
(1-p)^{n-k-j}
\left(p\sqrt{1-\varepsilon}\right)^j
\\
&\qquad=
\binom nk
\Pi_{\cH}(n)e^{k/(2\sigma^2)}
\left(1-p+p\sqrt{1-\varepsilon}\right)^{n-k}\,.
\end{align*}
where the last equality is the binomial theorem. 
\end{proof}

\begin{proof}[Proof of \Cref{lem:vc-population}]
The proof of \Cref{lem:population} with
\Cref{lem:vc-empirical,lem:vc-transfer} in place of
\Cref{lem:empirical,lem:transfer} gives
\begin{align*}
\E\bigl[\cD(\DIS(A_{n,r}))\bigr]
&\le
C\left(
\frac{\log\Pi_{\cH}(n)+\log n}{n}
+
\frac{\log n}{n}
\E\sum_{i=1}^n\1\{X_i\in\DIS(A_{n,r})\}
\right)
\\
\text{(\Cref{lem:vc-empirical})}\qquad
&\le
C\left(
\frac{\log\Pi_{\cH}(n)+\log n}{n}
+
\log n\sqrt{\frac{\log\Pi_{\cH}(n)}{n}}
\exp\left(\frac{r^2}{2\sigma^2}\right)
\right)
\\
&\le
C(1+\log n)
\sqrt{\frac{\log\Pi_{\cH}(n)}{n}}
\exp\left(\frac{r^2}{2\sigma^2}\right)\,,
\end{align*}
where the last inequality follows from
$\log 2\le\log\Pi_{\cH}(n)\le n\log 2$ and increasing the constant $C$.
\end{proof}

\begin{proof}[Proof of \Cref{thm:vc-extension}]
As in the proof of \Cref{thm:main}, \Cref{lem:regret-disagreement} gives
\begin{align*}
\E R_T
&\le
2\sigma
+
2\sum_{t=2}^T
\int_0^\infty
\Pp(G\ge s/\sigma)\,
\E\bigl[\cD(\DIS(A_{t-1,s+1}))\bigr]\,ds
\\
\text{(\Cref{lem:vc-population})}\quad
&\le
2\sigma
+
2\sum_{n=1}^{T-1}
\int_0^\infty
\Pp(G\ge s/\sigma)
\min\left\{
1,
C(1+\log n)
\sqrt{\frac{\log\Pi_{\cH}(n)}{n}}
\exp\left(\frac{(s+1)^2}{2\sigma^2}\right)
\right\}ds
\\
&\le
2\sigma
+
16C\log T\sqrt{\log\Pi_{\cH}(T)}
\sum_{n=1}^{T-1}\frac{1+\log n}{\sqrt n}
+
\frac{2(T-1)}{T^2}
\\
&=
O\!\left(
\sqrt{T\log\Pi_{\cH}(T)}\,\log^2T
\right)
\\
\text{(\Cref{lem:sauer})}\qquad
&=
O\!\left(
\sqrt{
T\VC(\mathcal H)
\log\left(1+\frac{T}{\VC(\mathcal H)}\right)
}\,\log^2T
\right)\,.
\end{align*}
\end{proof}

\input{contextual_bandit_proof}

%% file: regret-disagreement-lemma-proof.tex
\subsection{Stability via disagreement}

For a score function $S:\cH\to\R$ and $r\ge0$, write
$$A_r=\{h\in\mathcal H:S(h)\ge\max_{g\in\mathcal{H}} S(g)-r\}\,.$$ We first record the following simple fact:
\begin{fact}\label{fact:binary-disagreement}
For any $S:\cH\to\R$, $r\ge0$, and $A_r$ as above, assuming $\{h(X):h\in\mathcal H\}=\{0,1\}$ we have
\[
X\in\DIS(A_r)
\quad\Longleftrightarrow\quad
\left|
\max_{h\in\mathcal H:h(X)=1}S(h)
-
\max_{h\in\mathcal H:h(X)=0}S(h)
\right|\le r\,.
\]
\end{fact}

We will use the following binary stability bound.

\begin{lemma}\label{lem:binary-stability-disagreement}
Let $S:\cH\to\R$, $X\in\mathcal X$, $\ell:\{0,1\}\to[0,1]$,
$\sigma\ge\sqrt{2\pi}$, and $G\sim \mathcal{N}(0,1)$. For
$h^\star\in\argmax_{h\in\mathcal H}S(h)$, we have that
\begin{align*}
\ell(h^\star(X))
+\E_G\max_{h\in\mathcal H}
\{S(h)-\ell(h(X))+\sigma Gh(X)\}
-S(h^\star)
\le
2\int_0^\infty
\Pp(G\ge s/\sigma)\1\{X\in\DIS(A_{s+1})\}ds\,.
\end{align*}
\end{lemma}

\begin{proof}[Proof of \Cref{lem:regret-disagreement}]
Let
\(
L_t(h)=\sum_{s\le t}\ell_s(h(X_s))
\)
and, for $F:\cH\to\R$, define
\[
\Phi_t(F)
=
\E_{G_1,\dots,G_t\sim\mathcal{N}(0,1)}\max_{h\in\cH}
\left\{-F(h)+\sigma\sum_{s\le t}G_sh(X_s)\right\}.
\]
Since $\Phi_T(L_T)\ge-\min_{h\in\mathcal{H}}L_T(h)$ and $\Phi_0(L_0)=0$,
\begin{align*}
\E R_T
\le
\E\sum_{t=1}^T
\bigl[\ell_t(a_t)+\Phi_t(L_t)-\Phi_{t-1}(L_{t-1})\bigr].
\end{align*}

For $t=1$,
\begin{align*}
\ell_1(a_1)+\Phi_1(L_1)-\Phi_0(L_0)
\le
1+
\sigma\E_{G\sim\mathcal N(0,1)}|G|
=
1+\sigma\sqrt{\frac{2}{\pi}}
\le 2\sigma\,.
\end{align*}

Fix $t\ge2$, let $G_{t,1},\ldots,G_{t,t-1}\sim \mathcal{N}(0,1)$ be the perturbations in \Cref{alg:ftpl}, and draw, for the analysis, a fresh $G\sim \mathcal{N}(0,1)$. Denote
\(
S(h)=-L_{t-1}(h)+\sigma\sum_{i<t}G_{t,i}h(X_i)
\),
then \Cref{alg:ftpl} chooses
$h_t\in\argmax_{h\in\mathcal H}S(h)$ and we have
\begin{align*}
&\Phi_t(L_t)-\Phi_{t-1}(L_{t-1})
+\E_{G_{t,1},\ldots,G_{t,t-1}}\ell_t(a_t)
\\
&=
\E_{G_{t,1},\ldots,G_{t,t-1}}\ell_t(h_t(X_t))
\\
&\quad+
\E_{G_{t,1},\ldots,G_{t,t-1},G}
\max_{h\in\mathcal H}
\left\{
-L_t(h)
+\sigma\sum_{i<t}G_{t,i}h(X_i)
+\sigma Gh(X_t)
\right\}
\\
&\quad-
\E_{G_{t,1},\ldots,G_{t,t-1}}
\max_{h\in\mathcal H}
\left\{
-L_{t-1}(h)
+\sigma\sum_{i<t}G_{t,i}h(X_i)
\right\}
\\
&=
\E_{G_{t,1},\ldots,G_{t,t-1}}
\left[
\ell_t(h_t(X_t))
+\E_G\max_{h\in\mathcal H}
\{S(h)-\ell_t(h(X_t))+\sigma Gh(X_t)\}
-\max_{h\in\mathcal H}S(h)
\right]
\\
\text{(\Cref{lem:binary-stability-disagreement})}\qquad
&\le
2\int_0^\infty
\Pp(G\ge s/\sigma)\,
\E_{G_{t,1},\ldots,G_{t,t-1}}
\1\{X_t\in\DIS(A_{t-1,s+1})\}
\,ds\,.
\end{align*}
Since $A_{t-1,r}$ is independent of $X_t\sim\cD$,
\[
\E\,
\E_{G_{t,1},\ldots,G_{t,t-1}}
\1\{X_t\in\DIS(A_{t-1,r})\}
=
\E\bigl[\cD(\DIS(A_{t-1,r}))\bigr]\,.
\]
Summing over $t$ proves \Cref{lem:regret-disagreement}.
\end{proof}

\begin{proof}[Proof of \Cref{lem:binary-stability-disagreement}]
If for some $b\in\{0,1\}$, we have $h(X)=b$ for every $h\in\mathcal H$, then the left-hand side is zero and $X\notin\DIS(A_{s+1})$ for every $s\ge0$.
Otherwise, let $a=h^\star(X)$ and set
\[
\gamma
=
S(h^\star)-\max_{h\in\mathcal H:h(X)=1-a}S(h)\ge0.
\]
Then
\begin{align*}
&\ell(a)
+\E_G\max_{h\in\mathcal H}
\{S(h)-\ell(h(X))+\sigma Gh(X)\}
-\max_{h\in\mathcal H}S(h)
\\
&=
\E_G\max\{\sigma Ga,
-\gamma+\ell(a)-\ell(1-a)+\sigma G(1-a)\}
\\
&=
\E_G\left[
\sigma Ga+
\max\{\sigma G(1-2a)-\gamma+\ell(a)-\ell(1-a),0\}
\right]
\\
\text{(Gaussian symmetry)}\qquad
&=
\E_G\max\{\sigma G-\gamma+\ell(a)-\ell(1-a),0\}
\\
&\le
\E_G\max\{\sigma G-\gamma+1,0\}
\\
&\le
\E_G\max\{\sigma G-\max\{\gamma-1,0\},0\}
+\max\{1-\gamma,0\}
\\
&=
\int_{0}^\infty
\Pp(\max\{\sigma G-\max\{\gamma-1,0\},0\}\geq s)\,ds
+\max\{1-\gamma,0\}
\\&
=
\int_{\max\{\gamma-1,0\}}^\infty
\Pp(G\ge s/\sigma)\,ds+\max\{1-\gamma,0\}
\\
&\leq
2\int_{\max\{\gamma-1,0\}}^\infty
\Pp(G\ge s/\sigma)\,ds\,,
\end{align*}
where the last inequality follows from $\sigma\ge\sqrt{2\pi}$ and
\[
\max\{1-\gamma,0\}
\le
\1\{\gamma\le1\}\frac{\sigma}{\sqrt{2\pi}}
=
\1\{\gamma\le1\}\int_0^\infty \Pp(G\ge s/\sigma)\,ds
\le
\int_{\max\{\gamma-1,0\}}^\infty \Pp(G\ge s/\sigma)\,ds.
\]
Therefore, by \Cref{fact:binary-disagreement},
\begin{align*}
&\ell(h^\star(X))
+\E_G\max_{h\in\mathcal H}
\{S(h)-\ell(h(X))+\sigma Gh(X)\}
-\max_{h\in\mathcal H}S(h)
\\
&\qquad\le
2\int_{\max\{\gamma-1,0\}}^\infty
\Pp(G\ge s/\sigma)\,ds
\\
&\qquad=
2\int_0^\infty
\Pp(G\ge s/\sigma)\1\{\gamma\le s+1\}\,ds
\\
&\qquad=
2\int_0^\infty
\Pp(G\ge s/\sigma)
\1\{X\in\DIS(A_{s+1})\}\,ds\,.
\end{align*}
\end{proof}

%% file: empirical_lemma_proof.tex
\begin{proof}[Proof of \Cref{lem:empirical}]
Let $I$ index the observed contexts at which both labels are realizable, namely
\(I=\{i\in[n]:\{h(X_i):h\in\cH\}=\{0,1\}\}\).
For each $i\in I$, condition on
$G_1,\dots,G_{i-1},G_{i+1},\dots,G_n$ and for $b\in\{0,1\}$ write
\[
M_b
=
\max_{h\in\cH:h(X_i)=b}
\left\{
-\sum_{j=1}^n \ell_j(h(X_j))
+
\sigma\sum_{j\ne i}G_jh(X_j)
\right\}\,,
\]
and denote $\theta_i=(M_0-M_1)/\sigma$. Then
\begin{align*}
\1\{X_i\in\DIS(A_{n,r})\}
&=
\1\left\{
\max\{M_0,M_1+\sigma G_i\}
-
\min\{M_0,M_1+\sigma G_i\}
\le r
\right\}
\\
&=
\1\{|M_1+\sigma G_i-M_0|\le r\}
\\
&=
\1\{|G_i-\theta_i|\le r/\sigma\}\,.
\end{align*}

Writing $\varphi(g)=(2\pi)^{-1/2}e^{-g^2/2}$ and
\[
\widehat h_{G_1,\dots,G_n}
\in
\argmax_{h\in\cH}
\left\{
-\sum_{j=1}^n \ell_j(h(X_j))
+
\sigma\sum_{j=1}^nG_jh(X_j)
\right\}\,,
\]
we have that
\begin{align*}
\E_{G_1,\ldots,G_n} \sum_{i=1}^n\1\{X_i\in\DIS(A_{n,r})\}
&=
\sum_{i\in I}
\E_{G_1,\ldots,G_n}\1\{X_i\in\DIS(A_{n,r})\}
\\
&=
\sum_{i\in I}
\E_{G_1,\ldots,G_n}
\Pp_{G_i}\left(
X_i\in\DIS(A_{n,r})
\,\middle|\,
G_1,\dots,G_{i-1},G_{i+1},\dots,G_n
\right)
\\
&=
\sum_{i\in I}
\E_{G_1,\ldots,G_n}
\int_{\theta_i-r/\sigma}^{\theta_i+r/\sigma}
\varphi(g)\,dg
\end{align*}
and a change of variables gives
\begin{align*}
\E_{G_1,\ldots,G_n} \sum_{i=1}^n\1\{X_i\in\DIS(A_{n,r})\}
&=
\sum_{i\in I}
\E_{G_1,\ldots,G_n}\left[
\int_{\theta_i-r/\sigma}^{\infty}\varphi(g)\,dg
-
\int_{\theta_i+r/\sigma}^{\infty}\varphi(g)\,dg
\right]
\\
&=
\sum_{i\in I}
\E_{G_1,\ldots,G_n}
\int_{\theta_i}^{\infty}
\left[
\varphi(g-r/\sigma)-\varphi(g+r/\sigma)
\right]\,dg
\\
&=
2e^{-r^2/(2\sigma^2)}
\sum_{i\in I}
\E_{G_1,\ldots,G_n}
\int_{\theta_i}^\infty
\sinh\left(\frac r\sigma g\right)\varphi(g)\,dg
\\
&=
2e^{-r^2/(2\sigma^2)}
\sum_{i\in I}
\E_{G_1,\ldots,G_n}\left[
\1\{G_i\ge\theta_i\}
\sinh\left(\frac r\sigma G_i\right)
\right]\,.
\end{align*}
Now, observe that
\begin{align*}
\widehat h_{G_1,\dots,G_n}(X_i)
=
\1\{M_1+\sigma G_i\ge M_0\}
=
\1\left\{G_i\ge\frac{M_0-M_1}{\sigma}\right\}
=
\1\{G_i\ge\theta_i\}\,,
\end{align*}
and let $\xi_1,\ldots,\xi_n\sim\operatorname{Unif}\{-1,+1\}$. Since $\sinh(\cdot)$ is odd, we obtain
\begin{align*}
\E_{G_1,\ldots,G_n}\sum_{i=1}^n\1\{X_i\in\DIS(A_{n,r})\}
&=
2e^{-r^2/(2\sigma^2)}
\E_{G_1,\ldots,G_n}
\sum_{i\in I}
\sinh\left(\frac r\sigma G_i\right)
\widehat h_{G_1,\dots,G_n}(X_i)
\\
&\le
2e^{-r^2/(2\sigma^2)}
\E_{G_1,\ldots,G_n}
\max_{h\in\cH}
\sum_{i\in I}
\sinh\left(\frac r\sigma G_i\right)
h(X_i)
\\
&=
2e^{-r^2/(2\sigma^2)}
\E_{G_1,\ldots,G_n}
\E_{\xi_1,\ldots,\xi_n}
\max_{h\in\cH}
\sum_{i\in I}
\xi_i
\sinh\left(\frac r\sigma |G_i|\right)
h(X_i)
\\
\text{(\Cref{lem:massart})}\qquad
&\le
2e^{-r^2/(2\sigma^2)}
\E_{G_1,\ldots,G_n}
\sqrt{
2\log N
\sum_{i\in I}
\sinh^2\left(\frac r\sigma |G_i|\right)
}
\\
\text{(Jensen)}\qquad
&\le
2e^{-r^2/(2\sigma^2)}
\sqrt{
2n\log N\,
\E_G\sinh^2\left(\frac r\sigma G\right)
}
\\
&
=
2e^{-r^2/(2\sigma^2)}
\sqrt{
n\log N\,
\E_G\left[\cosh\left(\frac{2r}{\sigma}G\right)-1\right]
}
\\
&
=
2e^{-r^2/(2\sigma^2)}
\sqrt{
\frac{1}{2}n\log N\,
\left(\E_Ge^{\frac{2r}{\sigma}G}+\E_Ge^{-\frac{2r}{\sigma}G}-2\right)
}
\\
\text{(Gaussian MGF)}\qquad
&=
2e^{-r^2/(2\sigma^2)}
\sqrt{
n\log N
\left(e^{2r^2/\sigma^2}-1\right)
}
\\
&\le
2\sqrt{n\log N}
\exp\left(\frac{r^2}{2\sigma^2}\right)\,.
\end{align*}
\end{proof}

%% file: transfer_lemma_proof.tex
To prove \Cref{lem:transfer}, we use the following two technical results. \Cref{lem:suboptimal-set-characterization} gives a seemingly more complicated characterization of $A_{n,r}$, but exposes a useful partition over contexts. \Cref{lem:gaussian-holdout} uses a change-of-measure argument and the partitioned context set to bound the volume of contexts yielding disagreement.

\begin{lemma}
\label{lem:suboptimal-set-characterization}
Fix $B=\{b_1<\cdots<b_k\}\subseteq[n]$ and $h\in\cH$. Define
\(
P
=
\left\{i\notin B:(1-2h(X_i))G_i>\frac1\sigma\right\},
\)
and, for $w\in\R^k$,
\[
\cC_w
=
\left\{
g\in\cH:
g(X_i)=h(X_i)\ \forall i\in[n]\setminus(B\cup P),
\quad
\sum_{j=1}^k w_j\1\{g(X_{b_j})\ne h(X_{b_j})\}\ge-r
\right\}.
\]
Finally, let
\(
W_j
=
\bigl(\sigma G_{b_j}-\ell_{b_j}(1)+\ell_{b_j}(0)\bigr)
(1-2h(X_{b_j}))
\)
for $j\in[k]$. If $h\in A_{n,0}$ and it holds that 
$B=\{i:X_i\in\DIS(A_{n,r})\}$, then we have
\(
A_{n,r}=\cC_W.
\)
\end{lemma}

\begin{proof}
Every $g\in A_{n,r}$ agrees with $h$ outside $B$, and
\begin{align*}
&\sum_{i=1}^n
\left(
\ell_i(h(X_i))-\sigma G_i h(X_i)
-\ell_i(g(X_i))+\sigma G_i g(X_i)
\right)
\\
&=
\sum_{j=1}^k
\bigl(\sigma G_{b_j}-\ell_{b_j}(1)+\ell_{b_j}(0)\bigr)
(1-2h(X_{b_j}))
\1\{g(X_{b_j})\ne h(X_{b_j})\}
\\
&=
\sum_{j=1}^k
W_j\1\{g(X_{b_j})\ne h(X_{b_j})\}
\ge-r\,,
\end{align*}
where the last inequality follows from $h\in A_{n,0}$ and
$g\in A_{n,r}$. Thus $A_{n,r}\subseteq\cC_W$. Conversely, for
$g\in\cC_W$,
\begin{align*}
&\sum_{i=1}^n
\left(
\ell_i(h(X_i))-\sigma G_i h(X_i)
-\ell_i(g(X_i))+\sigma G_i g(X_i)
\right)
\\
&=
\sum_{j=1}^k
W_j\1\{g(X_{b_j})\ne h(X_{b_j})\}
\\
&\quad+
\sum_{i\in P}
\bigl(\sigma G_i-\ell_i(1)+\ell_i(0)\bigr)
(1-2h(X_i))
\1\{g(X_i)\ne h(X_i)\}
\\
&\ge-r\,,
\end{align*}
since, for every $i\in P$, we have
\(
\bigl(\sigma G_i-\ell_i(1)+\ell_i(0)\bigr)(1-2h(X_i))
>0\,.
\)
Hence $\cC_W\subseteq A_{n,r}$.
\end{proof}

\begin{lemma}
\label{lem:gaussian-holdout}
Fix $B=\{b_1<\cdots<b_k\}\subseteq[n]$ and $h\in\cH$, and let
$P,\cC_w,W$ be as in \Cref{lem:suboptimal-set-characterization}. Then, for
every $\varepsilon\in(0,1)$,
\[
\Pp\left(
\cD(\DIS(\cC_W))>\varepsilon,
\ X_i\notin\DIS(\cC_W)\ \forall i\in P
\,\middle|\,
P,(X_i)_{i\notin P}
\right)
\le
e^{k/(2\sigma^2)}(1-\varepsilon)^{|P|/2}\,.
\]
\end{lemma}

\begin{proof}
Let
\[
\mu_j
=
(1-2h(X_{b_j}))
\bigl(\ell_{b_j}(1)-\ell_{b_j}(0)\bigr),
\qquad j\in[k],
\]
so that $\|\mu\|^2\le k$. Conditional on
$(X_i,\ell_i)_{i=1}^n$ and $P$, we have
\(
W\sim \mathcal{N}(-\mu,\sigma^2I_k)
\).
Let $W_0\sim \mathcal{N}(0,\sigma^2I_k)$ be independent of everything else and define
\[
\mathcal E(w)
=
\left\{
\cD(\DIS(\cC_w))>\varepsilon,
\ X_i\notin\DIS(\cC_w)\ \forall i\in P
\right\}.
\]
Then
\begin{align*}
\Pp\left(
\mathcal E(W)
\,\middle|\,
P,(X_i)_{i\notin P}
\right)
&=
\E\left[
\E\left[
\1\{\mathcal E(W)\}
\,\middle|\,
P,(X_i,\ell_i)_{i=1}^n
\right]
\,\middle|\,
P,(X_i)_{i\notin P}
\right]
\\
&=
\E\left[
\int_{\R^k}
\1\{\mathcal E(w)\}
\frac{
\exp\left(-\frac{\|w+\mu\|^2}{2\sigma^2}\right)
}{
\exp\left(-\frac{\|w\|^2}{2\sigma^2}\right)
}
p_{W_0}(w)
\,dw
\,\middle|\,
P,(X_i)_{i\notin P}
\right]
\\
&=
\E\left[
\exp\left(
-\frac{\langle\mu,W_0\rangle}{\sigma^2}
-\frac{\|\mu\|^2}{2\sigma^2}
\right)
\1\{\mathcal E(W_0)\}
\,\middle|\,
P,(X_i)_{i\notin P}
\right]
\\
&\le
\sqrt{
\E\left[
\exp\left(
-\frac{2\langle\mu,W_0\rangle}{\sigma^2}
-\frac{\|\mu\|^2}{\sigma^2}
\right)
\,\middle|\,
P,(X_i)_{i\notin P}
\right]
}\cdot
\sqrt{
\Pp\left(
\mathcal E(W_0)
\,\middle|\,
P,(X_i)_{i\notin P}
\right)
}\,,
\end{align*}
where the last inequality is Cauchy--Schwarz. For the first term,
\begin{align*}
\E\left[
\exp\left(
-\frac{2\langle\mu,W_0\rangle}{\sigma^2}
-\frac{\|\mu\|^2}{\sigma^2}
\right)
\,\middle|\,
P,(X_i)_{i\notin P}
\right]
&=
\E\left[
\exp\left(-\frac{\|\mu\|^2}{\sigma^2}\right)
\E_{W_0}
\exp\left(
-\frac{2\langle\mu,W_0\rangle}{\sigma^2}
\right)
\,\middle|\,
P,(X_i)_{i\notin P}
\right]
\\
\text{(Gaussian MGF)}\qquad
&=
\E\left[
\exp\left(\frac{\|\mu\|^2}{\sigma^2}\right)
\,\middle|\,
P,(X_i)_{i\notin P}
\right]
\\
&\le
e^{k/\sigma^2}\,.
\end{align*}
For the second term, conditional on
$P,(X_i)_{i\notin P},W_0$, the set $\cC_{W_0}$ is fixed, while
$(X_i)_{i\in P}$ are i.i.d.\ from $\cD$. Hence
\begin{align*}
\Pp\left(
\mathcal E(W_0)
\,\middle|\,
P,(X_i)_{i\notin P}
\right)
&=
\E\left[
\1\left\{
\cD(\DIS(\cC_{W_0}))>\varepsilon
\right\}
\left(
1-\cD(\DIS(\cC_{W_0}))
\right)^{|P|}
\,\middle|\,
P,(X_i)_{i\notin P}
\right]
&\le
(1-\varepsilon)^{|P|}\,,
\end{align*}
and substituting these bounds above proves the claim.
\end{proof}

\begin{proof}[Proof of \Cref{lem:transfer}]
For each $B\subseteq[n]$ and $h\in\cH$, let $P,\cC_w,W$ be as in
\Cref{lem:suboptimal-set-characterization}. Then
\begin{align*}
&\Pp\!\left(
\sum_{i=1}^n\1\{X_i\in\DIS(A_{n,r})\}=k,
\ \cD(\DIS(A_{n,r}))>\varepsilon
\right)
\\
\text{(union-bound)}\qquad&\le
\sum_{\substack{B\subseteq[n]\\|B|=k}}
\sum_{h\in\cH}
\Pp\left(
h\in A_{n,0},
\ B=\{i:X_i\in\DIS(A_{n,r})\},
\ \cD(\DIS(A_{n,r}))>\varepsilon
\right)
\\
\text{(\Cref{lem:suboptimal-set-characterization})}\qquad
&\le
\sum_{\substack{B\subseteq[n]\\|B|=k}}
\sum_{h\in\cH}
\Pp\left(
\cD(\DIS(\cC_W))>\varepsilon,
\ X_i\notin\DIS(\cC_W)\ \forall i\in P
\right)
\\
&=
\sum_{\substack{B\subseteq[n]\\|B|=k}}
\sum_{h\in\cH}
\E\left[
\Pp\left(
\cD(\DIS(\cC_W))>\varepsilon,
\ X_i\notin\DIS(\cC_W)\ \forall i\in P
\,\middle|\,
P
\right)
\right]
\\
\text{(\Cref{lem:gaussian-holdout})}\qquad
&\le
\binom nk N e^{k/(2\sigma^2)}
\E\left[(1-\varepsilon)^{|P|/2}\right]
\\
\text{(binomial MGF)}\qquad
&=
\binom nk N e^{k/(2\sigma^2)}
\left(1-p+p\sqrt{1-\varepsilon}\right)^{n-k}\,,
\end{align*}
where $p=\Pp(G>1/\sigma)$ and the last equality uses
$|P|\sim\operatorname{Binomial}(n-k,p)$.
\end{proof}

%% file: contextual_bandit_proof.tex
\subsection{Contextual bandits proofs}
\label{sec:bandit-proofs}
Our analysis reduces \Cref{alg:bandit} to \Cref{alg:ftpl} on the exploration rounds. We thus show that \Cref{thm:main} holds for subsets of the horizon, in the following sense.
\begin{lemma}
\label{lem:bandit-short-ftpl}
Let $1\le m\le T$ and let $\mathcal F\subseteq\{0,1\}^\mathcal{X}$ be a finite binary class of size at
most $N$. \Cref{alg:ftpl} run for $m$ rounds with
$\sigma=4\sqrt{\log T}$ satisfies
\[
\E R_m
\le
C\sqrt{m\log N}\,\log^2T
\]
for a universal constant $C\ge1$.
\end{lemma}

\begin{proof}[Proof of \Cref{thm:bandit}]
For every $t\in[T]$, let
\[
h_t
\in
\argmin_{h\in\cH}
\sum_{s\in\mathcal E_{t-1}}
\bigl(c_s(A_s)-\sigma G_{t,s}\bigr)\1\{h(X_s)=A_s\},
\]
with the tie-breaking from \Cref{lem:bandit-reduction}. Write
$\mathcal E_T=\{t_1<\cdots<t_m\}$. Let
\(
Y_j=(X_{t_j},A_{t_j}),\,
\) and define a binary hypothesis class by \(
\mathcal F=\{f_h:h\in\cH\},
\) where $f_h:\mathcal{X}\times[K]\to\{0,1\}$ is defined as $f_h(X,A)=\1\{h(X)=A\}$. Also for any $j\in[m]$ define the binary losses $\ell^{\prime}_j:\{0,1\}\to[0,1]$ by 
\(
\ell^{\prime}_j(b)=c_{t_j}(A_{t_j})b.
\)
Conditional on $\mathcal E_T$, the variables $Y_1,\ldots,Y_m$ are i.i.d.\ from
$\cD\times\operatorname{Unif}([K])$.

Consider \Cref{alg:ftpl} run for $m$ rounds on $\mathcal F$, with contexts
$Y_j$, losses $\ell^{\prime}_j$, and perturbations $G_{t_j,t_i}$ on round
$j\in[m]$. With the same tie-breaking rule, its round-$j$ leader is
$f_{h_{t_j}}$, for
\[
h_{t_j}
\in
\argmin_{h\in\cH}
\left\{
\sum_{i<j}\ell^{\prime}_i(f_h(Y_i))
-
\sigma\sum_{i<j}G_{t_j,t_i}f_h(Y_i)
\right\}.
\]
Therefore, its regret is
\[
\widetilde R_m
=
\sum_{j=1}^m\ell^{\prime}_j(f_{h_{t_j}}(Y_j))
-
\min_{h\in\cH}\sum_{j=1}^m\ell^{\prime}_j(f_h(Y_j)),
\]
and, trivially for $m=0$ and by \Cref{lem:bandit-short-ftpl} otherwise,
\[
\E[\widetilde R_m\mid\mathcal E_T]
\le
C\sqrt{m\log N}\,\log^2T\,.
\]
For $t\in[T]$ and $h\in\cH$, let
\(
\widehat\ell_t(h)
=
\frac{K}{\gamma}c_t(A_t)
\1\{t\in\mathcal E_T\}
\1\{h(X_t)=A_t\}.
\)
Then \Cref{lem:bandit-reduction} gives
\begin{align*}
\E R_T^{\mathrm{bandit}}\le
\gamma T
+
\E\left[
\sum_{t=1}^T\widehat\ell_t(h_t)
-
\min_{h\in\cH}\sum_{t=1}^T\widehat\ell_t(h)
\right]
+
C\left(
\sqrt{\frac{KT\log N}{\gamma}}
+
\frac{K\log N}{\gamma}
\right).
\end{align*}
It remains to bound the middle term. For every $h\in\cH$, we have
\(
\widehat\ell_{t_j}(h)
=
\frac{K}{\gamma}\ell^{\prime}_j(f_h(Y_j)),
\) and \(
\widehat\ell_t(h)=0
\) for $t\notin\mathcal E_T$.
Therefore,
\begin{align*}
\E\left[
\sum_{t=1}^T\widehat\ell_t(h_t)
-
\min_{h\in\cH}\sum_{t=1}^T\widehat\ell_t(h)
\right]
&=
\frac{K}{\gamma}
\E\E[\widetilde R_m\mid\mathcal E_T]
\\
&\le
\frac{CK}{\gamma}
\E\sqrt{|\mathcal E_T|\log N}\,\log^2T
\\
\text{(Jensen)}\qquad
&\le
CK\sqrt{\frac{T\log N}{\gamma}}\,\log^2T\,.
\end{align*}

If $\gamma=1$, the result follows from $R_T^{\mathrm{bandit}}\le T$.
Otherwise, $\gamma^3T=K^2\log N\,\log^4T$, and
\begin{align*}
\E R_T^{\mathrm{bandit}}
&\le
\gamma T
+
CK\sqrt{\frac{T\log N}{\gamma}}\,\log^2T
+
C\left(
\sqrt{\frac{KT\log N}{\gamma}}
+
\frac{K\log N}{\gamma}
\right)
\\
&=
\gamma T
+
C\sqrt{\frac{K^2T\log N\,\log^4T}{\gamma}}
+
C\left(
\sqrt{\frac{K^2T\log N\,\log^4T}
{\gamma K\log^4T}}
+
\frac{K^2\log N\,\log^4T}
{\gamma K\log^4T}
\right)
\\
&=
\gamma T
+
C\gamma T
+
C\left(
\frac{\gamma T}{\sqrt K\log^2T}
+
\frac{\gamma^2T}{K\log^4T}
\right)
\\
&\le
C\gamma T
\\
&=
C K^{2/3}T^{2/3}(\log N)^{1/3}\log^{4/3}T.
\end{align*}
\end{proof}

\begin{proof}[Proof of \Cref{lem:bandit-short-ftpl}]
By \Cref{lem:regret-disagreement,lem:population},
\begin{align*}
\E R_m
&\le
2\sigma
+
2\sum_{n=1}^{m-1}\int_0^\infty
\Pp(G\ge s/\sigma)
\min\left\{
1,
C(1+\log n)\sqrt{\frac{\log N}{n}}
\exp\left(\frac{(s+1)^2}{2\sigma^2}\right)
\right\}
\,ds
\\
&\le
2\sigma
+
\sum_{n=1}^{m-1}\int_0^\infty
 e^{-s^2/(2\sigma^2)}
\min\left\{
1,
C(1+\log n)\sqrt{\frac{\log N}{n}}
\exp\left(\frac{(s+1)^2}{2\sigma^2}\right)
\right\}
\,ds
\\
&\le
2\sigma
+
\sum_{n=1}^{m-1}
\left[
C(1+\log n)\sqrt{\frac{\log N}{n}}
\int_0^{\sigma^2/2}2\,ds
+
\int_{\sigma^2/2}^\infty e^{-s^2/(2\sigma^2)}\,ds
\right]
\\
&\le
2\sigma
+
\sum_{n=1}^{m-1}
\left[
C(1+\log n)\sqrt{\frac{\log N}{n}}\,\sigma^2
+2e^{-\sigma^2/8}
\right]
\\
&\le
C\sqrt{m\log N}\,\log^2T.
\end{align*}
\end{proof}

We now turn to prove \Cref{lem:bandit-reduction}. We use the following standard concentration bound.

\begin{lemma}[Bernstein's inequality {\citep[Lemma~5]{seldin2012pac}}]
\label{lem:bandit-bernstein}
Let $D_1,\ldots,D_T$ be a martingale difference sequence (meaning that
$\E[D_t\mid D_1,\ldots,D_{t-1}]=0$ for every $t$) and suppose that
$D_t\le R$ almost surely. Let
\(
S_T=\sum_{t=1}^T D_t
\) be a corresponding martingale and \(
V_T=\sum_{t=1}^T\E[D_t^2\mid D_1,\ldots,D_{t-1}]
\) be the cumulative variance of this martingale.
Then for every fixed $\lambda\in[0,1/R]$,
\[
\E\exp\left(\lambda S_T-(e-2)\lambda^2V_T\right)\le1.
\]
\end{lemma}

We apply \Cref{lem:bandit-bernstein} to our problem in the following lemma.
\begin{lemma}
\label{lem:bandit-comparator-concentration}
The losses $\widehat\ell_t$ from \Cref{lem:bandit-reduction} satisfy
\[
\E\max_{h\in\cH}\sum_{t=1}^T
\bigl(\widehat\ell_t(h)-c_t(h(X_t))\bigr)
\le
C\left(
\sqrt{\frac{KT\log N}{\gamma}}
+
\frac{K\log N}{\gamma}
\right)
\]
for a universal constant $C\ge1$.
\end{lemma}

\begin{proof}
For fixed $h\in\cH$, set
\(
D_t=\widehat\ell_t(h)-c_t(h(X_t))
\) and write $\E_t[Y]=\E[Y\mid D_1,\ldots,D_{t-1}]$. Also write
\(
S_T=\sum_{t=1}^T D_t
\) and \(
V_T=\sum_{t=1}^T\E_tD_t^2.
\)
Then $D_t\le\widehat\ell_t(h)\le K/\gamma$ and
\begin{align*}
\E_tD_t
&=
\E_t\left[
\E\left[
D_t
\mid
D_1,\ldots,D_{t-1},X_t,c_t
\right]
\right]
\\
&=
\E_t\left[
\frac K\gamma c_t(h(X_t))
\Pp\left(
B_t=1,\ A_t=h(X_t)
\mid
D_1,\ldots,D_{t-1},X_t,c_t
\right)
-c_t(h(X_t))
\right]
\\
&=
\E_t\left[
\frac K\gamma c_t(h(X_t))
\cdot\gamma\cdot\frac1K
-c_t(h(X_t))
\right]
\\
&=0.
\end{align*}
In addition,
\begin{align*}
V_T
&=
\sum_{t=1}^T
\E_t\left[
\E\left[
D_t^2
\mid
D_1,\ldots,D_{t-1},X_t,c_t
\right]
\right]
\\
&=
\sum_{t=1}^T
\E_t\left[
\frac{\gamma}{K}
\left(\frac K\gamma-1\right)^2
c_t(h(X_t))^2
+
\left(1-\frac{\gamma}{K}\right)
c_t(h(X_t))^2
\right]
\\
&=
\sum_{t=1}^T
\E_t\left[
\left(\frac K\gamma-1\right)
c_t(h(X_t))^2
\right]
\\
&\le
\frac{KT}{\gamma}\,.
\end{align*}
Thus for every $\lambda\in[0,\gamma/K]$ we have that
\begin{align*}
\E e^{\lambda S_T}
&=
\E\left[
e^{(e-2)\lambda^2V_T}e^{\lambda S_T-(e-2)\lambda^2V_T}
\right]
\\
&\le
\exp\left((e-2)\lambda^2\frac{KT}{\gamma}\right)
\E e^{\lambda S_T-(e-2)\lambda^2V_T}
\\
&\le
\exp\left((e-2)\lambda^2\frac{KT}{\gamma}\right)\,,
\end{align*}
where the last inequality is \Cref{lem:bandit-bernstein}.
Hence by Jensen's inequality
\begin{align*}
\E\max_{h\in\cH}\sum_{t=1}^T
\bigl(\widehat\ell_t(h)-c_t(h(X_t))\bigr)
&\le
\frac1\lambda
\log\E\exp\left(
\lambda\max_{h\in\cH}\sum_{t=1}^T
\bigl(\widehat\ell_t(h)-c_t(h(X_t))\bigr)
\right)
\\
&\le
\frac1\lambda
\log\sum_{h\in\cH}
\E\exp\left(
\lambda\sum_{t=1}^T
\bigl(\widehat\ell_t(h)-c_t(h(X_t))\bigr)
\right)
\\
&\le
\frac{\log N}{\lambda}
+(e-2)\lambda\frac{KT}{\gamma}.
\end{align*}
Taking
\(
\lambda
=
\min\left\{
\gamma/K,
\sqrt{\gamma\log N/(KT)}
\right\}
\)
we have
\begin{align*}
\E\max_{h\in\cH}\sum_{t=1}^T
\bigl(\widehat\ell_t(h)-c_t(h(X_t))\bigr)
&\le
\frac{\log N}{\lambda}
+
(e-2)\lambda\frac{KT}{\gamma}
\\
&\le
\log N\left(
\frac{K}{\gamma}
+
\sqrt{\frac{KT}{\gamma\log N}}
\right)
+
(e-2)\sqrt{\frac{\gamma\log N}{KT}}\frac{KT}{\gamma}
\\
&\le
C\left(
\sqrt{\frac{KT\log N}{\gamma}}
+
\frac{K\log N}{\gamma}
\right).
\end{align*}
\end{proof}

\begin{proof}[Proof of \Cref{lem:bandit-reduction}]
First,
\begin{align*}
\min_{a\in[K]}V_t(a)
&=
\min_{\substack{a\in[K]\\ h\in\cH}}
\left\{
\1\{h(X_t)\ne a\}
+
\sum_{s\in\mathcal E_{t-1}}
\bigl(c_s(A_s)-\sigma G_{t,s}\bigr)\1\{h(X_s)=A_s\}
\right\}
\\
&=
\min_{h\in\cH}
\left\{
\min_{a\in[K]}\1\{h(X_t)\ne a\}
+
\sum_{s\in\mathcal E_{t-1}}
\bigl(c_s(A_s)-\sigma G_{t,s}\bigr)\1\{h(X_s)=A_s\}
\right\}
\\
&=
\min_{h\in\cH}
\sum_{s\in\mathcal E_{t-1}}
\bigl(c_s(A_s)-\sigma G_{t,s}\bigr)\1\{h(X_s)=A_s\}.
\end{align*}
Moreover,
\[
\argmin_{a\in[K]}V_t(a)
=
\left\{
h(X_t):
h\in
\argmin_{g\in\cH}
\sum_{s\in\mathcal E_{t-1}}
\bigl(c_s(A_s)-\sigma G_{t,s}\bigr)\1\{g(X_s)=A_s\}
\right\}.
\]
Hence
$A_t=h_t(X_t)$ whenever $B_t=0$, and therefore
\begin{align*}
\E R_T^{\mathrm{bandit}}
&=
\E\left[
\sum_{t=1}^T c_t(A_t)
-
\min_{h\in\cH}\sum_{t=1}^T c_t(h(X_t))
\right]
\\
&
\leq
\E\left[
\sum_{t=1}^TB_t+\sum_{t=1}^T c_t(h_t(X_t))
-
\min_{h\in\cH}\sum_{t=1}^T c_t(h(X_t))
\right]
\\
&=
\gamma T
+
\E\left[
\sum_{t=1}^T c_t(h_t(X_t))
-
\min_{h\in\cH}\sum_{t=1}^T c_t(h(X_t))
\right]\,.
\end{align*}

Furthermore,
\begin{align*}
\E\widehat\ell_t(h_t)
&=
\frac K\gamma
\E\left[
B_t c_t(A_t)\1\{h_t(X_t)=A_t\}
\right]
\\
&=
\frac K\gamma
\E\left[
\E\left[
B_t c_t(A_t)\1\{h_t(X_t)=A_t\}
\,\middle|\,
h_t,X_t,c_t
\right]
\right]
\\
&=
\frac K\gamma
\E\left[
\gamma\,
\E\left[
c_t(A_t)\1\{h_t(X_t)=A_t\}
\,\middle|\,
h_t,X_t,c_t,B_t=1
\right]
\right]
\\
&=
\frac K\gamma
\E\left[
\gamma\frac1K
\sum_{a\in[K]}
c_t(a)\1\{h_t(X_t)=a\}
\right]
\\
&=
\E\left[
\sum_{a\in[K]}
c_t(a)\1\{h_t(X_t)=a\}
\right]
\\
&=
\E c_t(h_t(X_t)).
\end{align*}
Thus
\begin{align*}
\E R_T^{\mathrm{bandit}}
&\le
\gamma T
+
\E\left[
\sum_{t=1}^T\widehat\ell_t(h_t)
-
\min_{h\in\cH}\sum_{t=1}^T c_t(h(X_t))
\right]
\\
&
=
\gamma T
+
\E\left[
\sum_{t=1}^T\widehat\ell_t(h_t)
-\min_{h\in\cH}\sum_{t=1}^T\widehat\ell_t(h)
+\min_{h\in\cH}\sum_{t=1}^T\widehat\ell_t(h)-
\min_{h\in\cH}\sum_{t=1}^T c_t(h(X_t))
\right]
\\
&\le
\gamma T
+
\E\left[
\sum_{t=1}^T\widehat\ell_t(h_t)
-
\min_{h\in\cH}\sum_{t=1}^T\widehat\ell_t(h)
\right]
+
\E\max_{h\in\cH}\sum_{t=1}^T
\bigl(\widehat\ell_t(h)-c_t(h(X_t))\bigr).
\end{align*}
The last term is bounded by \Cref{lem:bandit-comparator-concentration}, which completes the proof.
\end{proof}

%% file: discussion.tex
\section{Discussion}
We analyzed a simple FTPL algorithm for sequential binary prediction with stochastic contexts and adversarial losses, and gave the first nearly optimal expected regret guarantee using a single ERM call per round. We showed how the same guarantee also yields a $\widetilde O(K^{2/3}T^{2/3}(\log N)^{1/3})$ contextual-bandit regret bound, without sampling access to the context distribution. Our techniques are tailored to the classification setting, and extending \Cref{thm:main} to real-valued hypothesis classes remains open and may require modifications to the algorithm, the oracle model, or both.

\subsection*{Use of AI}
The result was discovered through an interactive process involving AI. The authors used GPT Astra to generate proof ideas and to assist with writing, both in prose and in mathematics. After extensive discussions with GPT Sol about related contextual-learning settings, on September 7, 2026, EH asked GPT Astra about the stochastic-context setting, and GPT Astra produced a proof of a result very similar to the one presented here. In particular, the original four-page AI-generated note already contained the key idea of exploiting the binary structure through disagreement analysis. Its argument, however, used a more complicated counting approach in place of the argument used here to prove \Cref{lem:transfer}, and this approach incurred additional logarithmic factors and did not extend to VC classes.

The arguments presented in the paper were subsequently developed through an interactive process between the authors and the AI, with both contributing substantive ideas, and the final proof differs substantially from the original generated note. Every mathematical argument in the final paper was independently rederived and checked by the human authors, who also wrote nearly all of the final exposition. The authors take full responsibility for the correctness of the results presented here.